\documentclass[conference]{IEEEtran}
\usepackage{times}

\usepackage[numbers]{natbib}
\usepackage{multicol}
\usepackage[bookmarks=true]{hyperref}

\usepackage{amsthm}
\usepackage{amsfonts}
\usepackage{mathptmx}       % selects Times Roman as basic font
\usepackage{helvet}         % selects Helvetica as sans-serif font
\usepackage{courier}        % selects Courier as typewriter font
\usepackage{upgreek} 
\usepackage{makeidx}         % allows index generation
\usepackage{graphicx}        % standard LaTeX graphics tool
\usepackage{multicol}        % used for the two-column index
\usepackage[bottom]{footmisc}% places footnotes at page bottom

\usepackage{epsfig}
\usepackage{amsmath}
\usepackage{amssymb}
\usepackage{subfig}
\usepackage{algorithm}
\usepackage{algorithmic}
\usepackage{color}
\newcommand{\todo}[1]{}
\usepackage[numbers]{natbib}
\usepackage[bookmarks=true]{hyperref}
\newtheorem{theorem}{Theorem}

\begin{document}

% paper title
\title{Predicting Contextual Sequences via Submodular Function Maximization}

% You will get a Paper-ID when submitting a pdf file to the conference system
\author{Author Names Omitted for Anonymous Review. Paper-ID 104}

% \author{\authorblockN{Debadeepta Dey}
% \authorblockA{The Robotics Institute\\
% Carnegie Mellon University\\
% Pittsburgh, USA\\
% Email: debadeep@ri.cmu.edu}
% \and
% \authorblockN{Tian Yu Liu}
% \authorblockA{The Robotics Institute\\
% Carnegie Mellon University\\
% Pittsburgh, USA\\
% Email: tommyliu@cmu.edu}
% \and
% \authorblockN{James Kirk\\ and Montgomery Scott}
% \authorblockA{Starfleet Academy\\
% San Francisco, California 96678-2391\\
% Telephone: (800) 555--1212\\
% Fax: (888) 555--1212}}

% avoiding spaces at the end of the author lines is not a problem with
% conference papers because we don't use \thanks or \IEEEmembership

% for over three affiliations, or if they all won't fit within the width
% of the page, use this alternative format:
% 
\author{\authorblockN{Debadeepta Dey\authorrefmark{1},
Tian Yu Liu\authorrefmark{1},
Martial Hebert\authorrefmark{1}, and
J. Andrew Bagnell\authorrefmark{1}
}
\authorblockA{\authorrefmark{1}The Robotics Institute\\
Carnegie Mellon University,
Pittsburgh, PA \\ Email: (debadeep, tianyul, hebert, dbagnell)@ri.cmu.edu}
}

\maketitle

%%%%%%%%%%%%%%%%%%%%%%%%%%%%%%%%%%%%%%%%%%%%%%%%%%%%%%%%%%%
\begin{abstract}
\emph{Sequence optimization}, where the items in a list are ordered to
maximize some reward has many applications such as web advertisement
placement, search, and control libraries in robotics. Previous work in sequence
optimization produces a static ordering that does not
take any features of the item or context of the problem into account.  In this
work, we propose a general approach to order the items within
the sequence based on the \emph{context} (\emph{e.g.}, perceptual information,
environment description, and goals). We take a simple, efficient,
reduction-based approach where
the choice and order of the items is established by repeatedly learning simple
classifiers or regressors for each ``slot'' in the sequence. Our approach leverages
recent work on submodular function maximization to provide
a formal regret reduction from submodular sequence optimization to
simple cost-sensitive prediction. We apply our contextual sequence prediction
algorithm to optimize control libraries and demonstrate results on two robotics
problems: manipulator trajectory prediction and mobile robot path planning.

% \emph{Control libraries} are a widely-used approach to tackle high-dimensional
% decision-making problems 
% in which a set of precomputed candidate control primitives are stored and then
% considered during execution. 
% During execution the library is used to either evaluate a sequence of
% control primitives to identify the
% best one, or to evaluate each one in sequence until the first satisfactory primitive
% is found.
% These methods play a crucial role in control problems, from state-of-the-art
% outdoor mobile robot navigation to
% dexterous manipulation. 
% This enables us to improve the speed and reliability of the control library
% method. The approach
% extends ideas of trajectory prediction from a single candidate action to a comparison between
% multiple options.
% Our proposed approach to contextual sequence prediction has application outside 
% our control domain to problems such as advertisement placement
% on web pages.
\end{abstract}

\IEEEpeerreviewmaketitle

%%%%%%%%%%%%%%%%%%%%%%%%%%%%%%%%%%%%%%%%%%%%%%%%%%%%%%%%%%%
\section{Introduction}
\label{introduction}
Optimizing the order of a set of choices is fundamental to many problems such as web
search, advertisement placements as well as in robotics and control.
\textit{Relevance} and \textit{diversity} are important properties of an
optimal ordering or sequence. 
In web search, for instance, if the search
term admits many different interpretations then the results should be
interleaved with items from each interpretation \cite{Radlinski08learningdiverse}. Similarly in
advertisement placement on web pages, advertisements should be chosen such
that within the limited screen real estate they are diverse yet relevant to
the page content. In robotics, \emph{control libraries} have the same
requirements for relevance and diversity in the ordering of member actions. In
this paper, we apply \emph{sequence optimization} to develop near-optimal
control libraries. In the context of control libraries, a sequence refers to
a ranked list of control action choices rather than a series of actions to be
taken. Examples of control actions include grasps for
manipulation, trajectories for mobile robot navigation or seed trajectories
for initializing a local trajectory optimizer.

\emph{Control libraries} are a collection of control actions obtained by
sampling a useful set of often high dimensional control trajectories or policies. Examples of control
\todo{discretization is technically correct, but might carry a different
  connotation then you mean-- i.e. brute force. -- FIXED}
libraries include a collection of feasible grasps for manipulation
\cite{chinellato2003}, a collection
of feasible trajectories for mobile robot navigation
\cite{green2006paths}, and a collection of expert-demonstrated trajectories for a walking robot (Stolle
et. al.\ \cite{stolle2006policies}). Similarly, recording demonstrated
trajectories of experts aggressively flying unmanned aerial vehicles (UAVs)
has enabled dynamically feasible trajectories to be quickly generated by
concatenating a suitable subset of stored trajectories in the control
library \cite{frazzoli2000rhc}.
\todo{Finish above sentence and relate to below -- FIXED ABOVE}

Such libraries are an effective means of spanning the
space of all feasible control actions while at the same time dealing with
computational constraints.
The performance of control libraries on the specified task can be
significantly improved by careful consideration of the \emph{content} and
\emph{order} of actions in the library. To make this clear let us
consider specific examples:

\textbf{Mobile robot navigation.}  In mobile robot navigation the task is to find a collision-free, low cost of
traversal path which leads to the specified goal on a map. Since sensor
horizons are finite and robots usually have constrained motion models and
non-trivial dynamics, a library of trajectories respecting the dynamic and
kinematic constraints of the robot are precomputed and stored in memory.
This constitutes the control library. It is desired to
sample a subset of trajectories at every time step so that the overall cost
of traversal of the robot from start to goal is minimized.

\textbf{Trajectory optimization.}  Local trajectory optimization techniques are sensitive to initial trajectory seeds. Bad
trajectory initializations may lead to slow optimization, suboptimal
performance, or even remain in collision. 
In this setting, the control actions are end-to-end trajectory seeds that act
as input to the optimization.
Zucker \cite{zucker2009proposal} and Jetchev et al.\ \cite{jetchev2009trajectory}
proposed methods for predicting trajectories from a
precomputed library using features of the environment, yet these methods do
not provide recovery methods when the prediction fails. Having a
\emph{sequence} of initial trajectory seeds provides fallbacks should earlier
ones fail.

\todo{This sentence is good, but it's unconnected to the later discussion about predicting sequences. Perhaps give the examples first,
discuss methods of selecting libraries, then explain contextual one-off
prediction has been done. (THe above sentence.) Then go into your attempted
contribution. -- FIXED}

\textbf{Grasp libraries.}  During selection of grasps for an object, a library of feasible grasps can be
evaluated one at a time until a collision-free, reachable grasp is found.\todo{Long sentences. What makes a grasp successful? Working in practice
or being collision free and reachable by the robot? -- FIXED}
While a naive ordering of grasps can be based on force closure and
stability criteria \cite{berenson2007}, if a grasp fails, then grasps similar
to it are also likely to fail. A more principled ordering approach
which takes into account features of the environment can reduce depth of the
sequence that needs to be searched by having diversity in higher ranked grasps.

\todo{I'd put the mobile robot first. -- FIXED}

Current state-of-the-art methods in the problems we address either predict only a
single control action in the library that has the highest score for the current environment, or use an ad-hoc ordering of actions
such as random order or by past rate of success.
If the predicted action fails then systems (e.g. manipulators and autonomous vehicles) are unable to 
recover or have to fall back on some heuristic/hard-coded contingency
plan. Predicting a \emph{sequence} of options to evaluate is
necessary for having intelligent, robust behavior. Choosing the order of
evaluation of the actions based on the context of the environment leads to
more efficient performance.

\todo{It's unclear what the sequence means. You mean that each option can be evaluated, but it could
mean as it stands that a sequence is actually executed on the system. Further, the last sentence is true, but perhaps not what
you mean:
Crucially, any approach for
predicting multiple options must take into account the context to be robust
to environmental changes.

All control library methods take into account context and hence handle environmental changes. Contextual libraries choose the order and elements
based on this context to be much more efficient.
-- MADE CLEAR IN THE FIRST PARAGRAPH AND CHANGED HERE
}

\todo{Perhaps the below should be sub-section. Achieving contextual
  optimization? Should it be after you define the problem and sub-modularity
  or some such? -- CONTEST- the contribution needs to be before the math and
  we need the below for the contribution to make sense.}

A naive way of predicting contextual sequences would be to train a multi-class
classifier over the label space consisting of all possible sequences of a
certain length. This space is exponential in the number of classes and
sequence length posing information theoretic difficulties. A more reasonable
method would be to use the greedy selection technique by Steeter
et al. \cite{streeter2007online} over the hypothesis space of all predictors which is guaranteed to yield sequences
within a constant factor of the optimal sequence. Implemented naively, this remains expensive as it
must explicitly enumerate the label space. Our simple reduction based approach where we
propose to train multiple multi-class classifiers/regressors to mimic greedy selection
given features of the environment is both efficient and maintains performance guarantees of the greedy selection.

Perception modules using \todo{This paragraph might be coming late? -- FIXED
  moved from below to here}
sensors such as cameras and lidars are part and parcel of modern robotic
systems. Leveraging such information in addition to the feedback of success or
failure is \emph{conceptually} straightforward: instead of considering a sequence of control actions, we consider
a sequence of classifiers which map features $X$ to control actions $A$, and attempt to find the best
such classifier at each slot in the control action sequence.
By using contextual features, our method has the benefit
of closing the loop with perception while maintaining the performance guarantees in Streeter et al.\cite{streeter2007online}.

The outlined examples present loss functions that depend only on the ``best''
action in the sequence, or attempt to minimize the prediction depth to find a
satisfactory action. Such loss functions are monotone, submodular -- \textit{i.e.}, one
with diminishing returns.\footnote{For more information on submodularity and optimization of submodular functions we refer readers to
the tutorial \cite{guestrin08submodtut}.} We define these functions in section \ref{contextual_optimization_of_control_libraries} and review the online
submodular function maximization approach of Streeter et al.\ \cite{streeter2007online}. We also describe our contextual sequence optimization (\textsc{ConSeqOpt}) algorithm in detail. Section \ref{experiments}
shows our algorithm's performance improvement over alternatives for local trajectory optimization
for manipulation  and in path planning for mobile robots.

%THE BACKGROUND SECTION IS MOVED TO CONTEXTUAL_OPTIMIZATION SECTION

Our contributions in this work are:
\begin{itemize}
\item We propose a simple, near-optimal reduction for contextual sequence
  optimization. Our approach moves from predicting a single decision based on
  features to making a \emph{sequence} of predictions, a problem that arises
  in many domains including advertisement prediction \cite{streeter2009online,Radlinski08learningdiverse}
  and search. 
\todo{Footnote seems excessive. Just cite both? -- FIXED }
\item The application of this technique to the contextual optimization of
  control libraries. We demonstrate the efficacy of the approach on two
  important problems: robot manipulation planning and mobile robot
  navigation. Using the sequence of actions generated by our approach we
  observe improvement in performance over sequences generated by either random
  ordering or decreasing rate of success of the actions.
\item Our algorithm is generic and can be naturally
  applied to any problem where ordered sequences (e.g.,
  advertisement placement, search, recommendation systems, etc) need to be
  predicted and relevance and diversity are important.
\todo{Ordering of final sentence is unnatural. Rephrase. -- FIXED}
\end{itemize}

%%%%%%%%%%%%%%%%%%%%%%%%%%%%%%%%%%%%%%%%%%%%%%%%%%%%%%%%%%%
\section{Contextual Optimization of Sequences}
\label{contextual_optimization_of_control_libraries}

% NEW SECTION FROM INTRO
\subsection{Background}
\label{subsection.background}

\todo{Input space to what? Decide if you really think this belongs above the
  problem definition. -- CONTEST}
The control library is a set $\mathcal{V}$ of actions. Each action is denoted by $a \in \mathcal{V}$.
\footnote{In this work we assume that each action choice takes the same time to execute
  although the proposed approach can be readily extended to handle different
  execution times.}
Formally, a function
$f : \mathcal{S} \rightarrow \Re_+ $ is monotone
submodular for any sequence $S \in \mathcal{S}$ where $ \mathcal{S} $ is the
set of all sequences of actions if it satisfies the following two properties:
\begin{itemize}
\item (Monoticity) for any sequence $S_{1}, S_{2} \in \mathcal{S}$, $f(S_{1})
  \leq f(S_{1} \oplus S_{2})$ and $f(S_{2}) \leq f(S_{1} \oplus S_{2})$
\item (Submodularity) for any sequence $S_{1}, S_{2} \in \mathcal{S}$,
  $f(S_{1})$ and any action $a \in \mathcal{V}$,
  $f(S_{1} \oplus S_{2} \oplus \langle a \rangle) - f(S_{1} \oplus S_{2}) \leq f(S_{1}
  \oplus \langle a \rangle) - f(S_{1})$
\end{itemize}
where $\oplus$ denotes order dependent concatenation of sequences. These
imply that the function always increases as more actions are added
to the sequence (monotonicity) but the gain obtained by adding an action to a larger
pre-existing sequence is less as compared to addition to a smaller
pre-existing sequence (sub-modularity).

For control library optimization, we attempt to optimize one of two possible
criteria: the cost of the best action $a$ in a sequence (with a budget on sequence size) or the time (depth in sequence) 
to find a satisficing action. For the former, we consider the function,
\begin{equation}
  \label{f_path_planning}
  f \equiv \frac{N_o - \min(\textnormal{cost}(a_1),\textnormal{cost}(a_2),\ldots,\textnormal{cost}(a_N))}{N_o},
\end{equation}
where $\textnormal{cost}$ is an arbitrary cost on an action (${a_i}$) given an
environment and $N_o$ is a constant, positive normalizer which is the highest
\todo{Why non-negative? It better be positive. The highest-->largest. -- FIXED }
cost. \footnote{For mobile robot path planning, for instance, $cost(a_i)$ is
  typically a simple measure of mobility penalty based on terrain for traversing
  a trajectory $a_i$ sampled from a set of trajectories and terminating in a
  heuristic cost-to-go estimate, compute by, e.g. A*.} Note that the $f$ takes in as arguments the
sequence of actions $a_1, a_2, \ldots, a_N$ directly, but is also implicitly dependent on
the current environment on which the actions are evaluated in $cost(a_i)$.
Dey et al.\ \cite{dey2011efficientcontrol} prove that this criterion is monotone submodular in sequences of control
actions and can be maximized-- within a constant factor-- by greedy
approaches similar to Streeter et al.\
\cite{streeter2007online}. 

For the latter optimization criteria, which arises in grasping and trajectory
seed selectin, we define the monotone, sub-modular loss function
$f:\mathcal{S} \rightarrow [0,1]$ as $f \equiv P(S)$ where $P(S)$ is the
probability of successfully grasping an object in a given scenario using the
sequence of grasps provided. It is easy to check \cite{dey2011efficientcontrol} that this function is also monotone and submodular,
as the probability of success always increases as we consider additional elements. Minimizing the depth in the control
library to be evaluated becomes our goal. In the rest of the paper all
objective functions are assumed to be monotone submodular unless noted
otherwise. \todo{All functions? Maybe the objective functions, but certainly
  not all functions. -- FIXED}

While optimizing these over library actions is effective, the ordering of
actions does not take into account the current context.
\todo{Well, the evaluation does. The library doesn't. -- FIXED}
People do not attempt to grasp objects based only on previous performance of grasps: they take into account the position,
orientation of the object, the proximity and arrangement of clutter around the object and also
their own position relative to the object in the current environment. 
\todo{Hard to introspect. Perhaps humans don't do contextual prediction at
  all. Just check *a lot* of options against perceptual data. Careful about
  this claim. -- FIXED by removing}

\subsection{Our Approach}
\label{subsection.ourapproach}

We consider functions that are submodular over sequences of either control
actions themselves or, crucially, over classifiers that take as input
environment features $\mathbf{X}$ and map to control actions
$\mathcal{V}$. Additionally, by considering many environments, the expectation
of $f$ in equation \eqref{f_path_planning} over these environments also maintains these properties. In our work, we always
consider the expected loss averaged over a (typically empirical) distribution of environments. \todo{Note that our objective
function has the right properties over classifiers as well as ground actions. }

In Algorithm \ref{conseqopt.classification}, we present a simple
approach for learning such a near-optimal contextual
control library.
\subsection{Algorithm for Contextual Submodular Sequence Optimization}
\begin{figure*}[ttt!]
  \begin{minipage}[t]{\textwidth}
    \begin{algorithm}[H]
      \caption{\texttt{Algorithm for training \textsc{ConSeqOpt} using
          classifiers}}
      \label{conseqopt.classification}
      \begin{algorithmic}[1]
        \REQUIRE \texttt{sequence length N, multi-class cost sensitive classifier routine
          $\mathbf{\pi}$, dataset $D$ of $|D|$ number of environments and associated features $\mathbf{X}$,
          library of control actions $\mathcal{V}$}
        \ENSURE \texttt{sequence of classifiers $\mathbf{\pi_1}, \mathbf{\pi_2}, \ldots, \mathbf{\pi_N}$}
        \FOR{$i = 1$ \TO $N$} \label{class.line1}
        \STATE $\mathbf{M_{L_i}} \leftarrow \texttt{computeTargetActions}(\mathbf{X},
        \mathbf{Y_{\pi_1,\pi_2,\ldots,\pi_{i-1}}}, \mathcal{V})$ \label{class.line2}
        \STATE $\mathbf{\pi_i} \leftarrow \texttt{train}(\mathbf{X}, \mathbf{M_{L_{i}}})$ \label{class.line3}
        \STATE $\mathbf{Y_{\pi_i}} \leftarrow \texttt{classify}(\mathbf{X})$ \label{class.line4}
        \ENDFOR \label{class.line5}
      \end{algorithmic}
    \end{algorithm}
  \end{minipage}
  \vfill
  \begin{minipage}[t]{\textwidth}
    \begin{algorithm}[H]
      \caption{\texttt{Algorithm for training \textsc{ConSeqOpt} using
          regressors}}
      \label{conseqopt.regression}
      \begin{algorithmic}[1]
        \REQUIRE \texttt{sequence length N, regression routine
          $\Re$, dataset $D$ of $|D|$ number of environments,
          library of control actions $\mathcal{V}$}
        \ENSURE \texttt{sequence of regressors $\Re_1, \Re_2, \ldots, \Re_N$}
        \FOR{$i = 1$ \TO $N$} \label{reg.line1}
        \STATE $\mathbf{X_i}, \mathbf{M_{B_i}} \leftarrow \texttt{computeFeatures\&Benefit}(D, \mathbf{Y_{\Re_1,\Re_2,\ldots,\Re_{i-1}}}, \mathcal{V})$ \label{reg.line2}
        \STATE $\Re_i \leftarrow \texttt{train}(\mathbf{X_i}, \mathbf{M_{B_{i}}})$ \label{reg.line3}
        \STATE $\mathbf{\tilde{M}_{B_i}} \leftarrow \texttt{regress}(\mathbf{X_i}, \Re_i)$ \label{reg.line4}
        \STATE $\mathbf{Y_{\Re_i}} = \texttt{argmax}(\mathbf{\tilde{M}_{B_i}})$ \label{reg.line5}
        \ENDFOR \label{reg.line6}
      \end{algorithmic}
    \end{algorithm}
  \end{minipage}
  \hfill
\end{figure*}

\begin{figure}[ht]
  \begin{center}
    \includegraphics[width=0.38\textwidth]{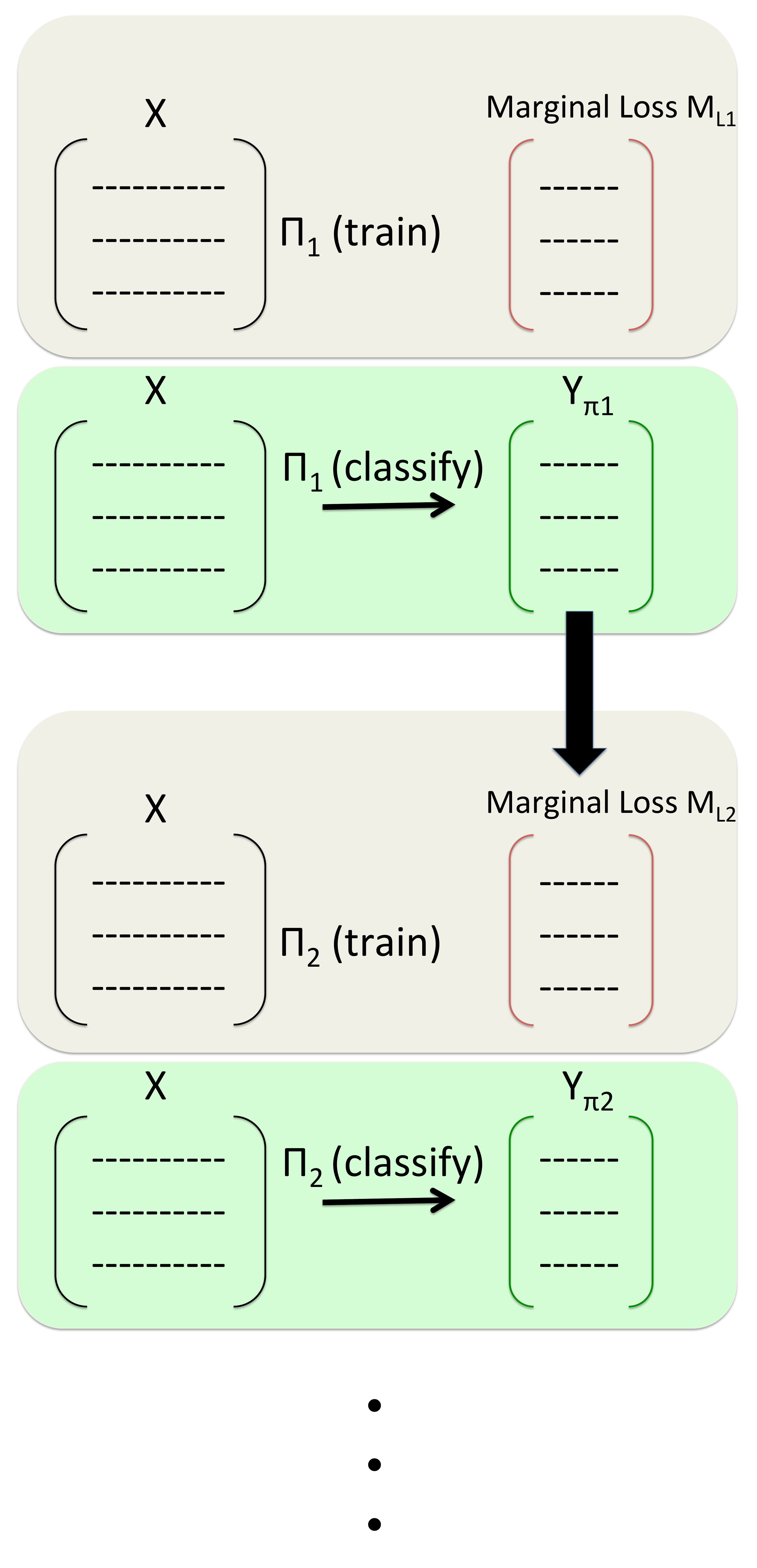}
    \caption{Schematic of training sequence of classifiers for regret reduction of
      contextual sequence optimization to multi-class, cost-sensitive, classification}
    \label{fig.schematic}
  \end{center}
\end{figure}

Figure \ref{fig.schematic} shows the schematic diagram for algorithm
\ref{conseqopt.classification} which trains a classifier for each slot of the
sequence. Define matrix $\mathbf{X}$ to be the set of 
features from a distribution of example environments (one feature vector per
row) and matrix $\mathbf{Y}$ to be the corresponding target
action identifier for each example. Let each feature vector contain $L$
attributes. Let $D$ be the set of example environments
containing $|D|$ examples. The size of $\mathbf{X}$ is $|D|
\times L$ and size of $\mathbf{Y}$ is $|D| \times 1$. We denote the
$i^{th}$ classifier by $\mathbf{\pi_i}$. Define $\mathbf{M_{L_i}}$ to be the matrix of
marginal losses for each environment for the $i^{th}$ slot of the sequence. In
the parlance of cost-sensitive learning $\mathbf{M_{L_i}}$ is the example-dependent
cost  matrix. $\mathbf{M_{L_i}}$ is of dimension $|D| \times
|\mathcal{V}|$. Each row of $\mathbf{M_{L_i}}$ contains, for the corresponding
environment, the loss suffered by the classifier for selecting a particular
action $a \in \mathcal{V}$. The most beneficial action has $0$ loss while
others have non-zero losses. These losses are normalized to be within
$[0-1]$. We detail how to calculate the entries of
$\mathbf{M_{L_i}}$ below. Classifier inputs are the set of feature vectors $\mathbf{X}$ for the dataset
of environments and the marginal loss matrix $\mathbf{M_{L_i}}$. 

For ease of understanding let us walk through the training of the first two
classifiers $\pi_1$ and $\pi_2$. 

Consider the first classifier training in Figure \ref{fig.schematic} and its inputs $\mathbf{X}$ and
$\mathbf{M_{L_1}}$. Consider the first row of $\mathbf{M_{L_1}}$. Each element of this row
corresponds to the loss incurred if the corresponding action in $\mathcal{V}$
were taken on the corresponding environment whose features are in the first
row of $\mathbf{X}$. The best action has $0$ loss while all the others have relative
losses in the range $[0-1]$ depending on how much worse they are compared to
the best action. This way the rest of the rows in $\mathbf{M_{L_1}}$ are filled out.
The cost sensitive classifier $\mathbf{\pi_1}$ is trained. The set of features  $\mathbf{X}$
from each environment in the training set are again presented to it to
classify. The output is matrix $\mathbf{Y_{\pi_1}}$ which contains the selected action
for the $1^{st}$ slot for each environment. As no element of the hypothesis
class performs perfectly this results in $\mathbf{Y_{\pi_1}}$, where not every
environment had the $0$ loss action picked. 

Consider the second classifier training in Figure
\ref{fig.schematic}. Consider the first row of
$\mathbf{M_{L_2}}$. Suppose control action id $13$ was selected by classifier $\pi_1$ in the
classification step for the first environment, which provides a gain of $0.6$ to the objective function
$f$ i.e. $f[13] = 0.6$. For each of the control actions $a$ present in the
library $\mathcal{V}$ find the action
which provides maximum marginal improvement i.e.\ $a_{max} = argmax_a(f([13, a]) -
f([13])) = argmax_a(f([13, a]) - 0.6$. Additionally
convert the marginal gains computed for each $a$ in the
library to proportional losses and store in the first row of
$\mathbf{M_{L_2}}$. If $a_{max}$ is the action with the maximum marginal gain
then the loss for each of the other actions is $f([13, a_{max}]) -
f([13,a])$. $a_{max}$ has $0$ loss while other actions have $>=0$ loss. The 
rest of the rows are filled up similarly. $\pi_2$ is trained, and evaluated on
same dataset to produce $Y_{\pi_2}$. \todo{Careful with the phrasing of the
  last sentence. IT sounds like an overfitting nightmare. Train and test on
  same set? -- FIXED}

This procedure is repeated for all $N$ slots
producing a sequence of classifiers
$\mathbf{\pi_1},\mathbf{\pi_2},\ldots,\mathbf{\pi_N}$. The idea is that a
classifier must suffer a high loss when it chooses a
control action which provides little marginal gain when a 
higher gain action was available. Any cost-sensitive multi-class classifier may be used.

During test time, for a given environment features are extracted, and the
classifiers associated  with each slot of the sequence outputs a control
action to fill the slot. This sequence can then be evaluated as usual.
This procedure is formalized in Algorithm \ref{conseqopt.classification}. In
$\texttt{computeTargetActions}$ the
previously detailed procedure for calculating the entries of the marginal loss
matrix $M_{L_i}$ for the $i^{th}$ slot is carried out, followed by the
training step in $\texttt{train}$ and classification step in $\texttt{classify}$.

Algorithm \ref{conseqopt.regression} has a similar structure as
algorithm \ref{conseqopt.classification}. This alternate formulation has the advantage of being able to 
add actions to the control library without retraining the sequence
of classifiers. Instead of directly identifying a target
class, we use a squared-loss regressor in each slot to produce an estimate of the marginal
benefit from each action at that particular slot. Hence $\mathbf{M_{B_{i}}}$ is a
$|D| \times |\mathcal{V}|$ matrix of the actual marginal benefit computed in a similiar
fashion as $\mathbf{M_{L_i}}$ of Algorithm \ref{conseqopt.classification}, and
$\mathbf{\tilde{M}_{B_i}}$
is the estimate given by our regressor at $i^{th}$ slot. In line \ref{reg.line2} we
compute the feature matrix $\mathbf{X_i}$. In this case, a feature vector is computed \emph{per
action} per environment, and uses information from the previous slots'
target choice $\mathbf{Y_{\Re_i}}$. For feature vectors of length $L$, $\mathbf{X_i}$ has
dimensions $|D||\mathcal{V}| \times L$. The features and marginal benefits at
$i^{th}$ slot are used to train regressor $\mathbf{\Re_i}$, producing the
estimate $\mathbf{\tilde{M}_{B_i}}$. We then
pick the action $a$ which produces the maximum $\mathbf{\tilde{M}_{B_i}}$ to be our
target choice $\mathbf{Y_{\Re_i}}$, a $|D|$ length vector of indices into $\mathcal{V}$ for each environment.

\subsection{Reduction Argument}
We establish a formal regret reduction \cite{beygelzimer2005error} between
cost sensitive multi-class classification error and the resulting error on the
learned sequence  of classifiers.
Specifically, we demonstrate that if we consider the control
actions to be the classes and train a series of classifiers-- one for each
slot of the sequence-- on the features of a distribution of environments then
we can produce a near-optimal sequence of classifiers. This sequence of
classifiers can be invoked to approximate the greedy sequence
constructed by allowing additive error in equation \eqref{approx_greedy}. 
\todo{It actually approximates the optimal sequence of classifiers, not just
  the optimal greedy sequence! -- FIXED}

\begin{theorem}
  \label{classification_proof}
  If each of the classifiers ($\mathbf{\pi_i}$) trained in
  Algorithm \ref{conseqopt.classification} achieves multi-class cost-sensitive regret of $r_i$, then the
  resulting sequence of classifiers is within
  at least $(1-\frac{1}{e}) \max_{S \in
    \mathcal{S}}{f(S)} - \sum_{i=1}^Nr_i$ of the optimal such  sequence of
  classifiers $S$ from the same hypothesis space. \footnote{When the objective is
    to minimize the time (depth in sequence) to find a satisficing element then
    the resulting sequence of classifiers $f(\hat{S}_{\langle N \rangle}) 
    \leq
    4 \int_0^\infty 1 - \max_{S \in \mathcal{S}}{f(S_{\langle n \rangle})}dn +
    \sum_{i=1}^N r_i$.}

\end{theorem}
\begin{proof} \emph{(Sketch)}
  Define the loss of
  a multi-class, cost-sensitive classifier $\mathbf{\pi}$ over a distribution of
  environments $D$ as $l(\mathbf{\pi}, D)$. Each example can
  be represented as $(x_n, m_n^1, m_n^2, m_n^3, \ldots, m_n^{|\mathcal{V}|})$
  where $x_n$ is the set of features representing the $n^{th}$ example
  environment and $m_n^1, m_n^2, m_n^3, \ldots, m_n^{|\mathcal{V}|}$ are the per
  class costs of misclassifying $x_n$. $m_n^1, m_n^2, m_n^3, \ldots,
  m_n^{|\mathcal{V}|}$ are simply the $n^{th}$ row of $M_{L_{i}}$ (which
  corresponds to the $n^{th}$ environment in the dataset $D$). The
  best class has a $0$ misclassification cost and while others are greater than
  equal to $0$ (There might be multiple actions which will yield equal marginal
  benefit). Classifiers generally minimize the expected loss $l(\mathbf{\pi},
  D) = \underset{{(x_n, m_n^1, m_n^2, m_n^3, \ldots, m_n^{|\mathcal{V}|})
      \sim D}}{\mathbb{E}}[C_{{\pi}(x_n)}]$ where $C_{{\pi}(x_n)} =
  m^{\pi(x_n)}_n$ denotes the
  example-dependent multi-class misclassification cost. The best classifier in
  the hypothesis space $\Uppi$ minimizes $l(\mathbf{\pi}, D)$
  \begin{equation}
    \mathbf{\pi^*} = \underset{\mathbf{\pi} \in \Uppi}{\operatorname{argmin}}
    \underset{{(x_n, m_n^1, m_n^2, m_n^3, \ldots, m_n^{|\mathcal{V}|})
        \sim D}}{\mathbb{E}}[C_{{\pi}(x_n)}]
  \end{equation}

  The regret of
  $\mathbf{\pi}$ is defined as $r =
  l(\mathbf{\pi}, D) - l(\mathbf{\pi^*}, D)$.
  Each classifier associated with $i^{th}$ slot of the sequence has a regret $r_i$. 

  Streeter et al.\ \cite{streeter2007online} consider the case where the
  $i^{th}$ decision made by the greedy algorithm is performed with additive error
  $\epsilon_i$. Denote by
  $\hat{S} = \langle \hat{s}_1, \hat{s}_2, \ldots, \hat{s}_N \rangle$ a variant
  of the sequence $S$ in which the $i^{th}$ argmax is evaluated with additive error
  $\epsilon_i$. This can be formalized as 
  \begin{equation}
    \label{approx_greedy}
    f(\hat{S}_i \oplus \hat{s}_i) - f(\hat{S}_i) \geq \max_{s_i \in
      \mathcal{V}}{f(\hat{S}_i \oplus s_i) - f(\hat{S}_i)} -\epsilon_i
  \end{equation}
  where $\hat{S}_0 = \langle \rangle$, $\hat{S}_i = \langle \hat{s}_{1},
  \hat{s}_{2}, \hat{s}_{3}, \ldots, \hat{s}_{i-1} \rangle$ for $i \geq 1$ and
  $s_i$ is the predicted control action by classifier $\mathbf{\pi_i}$. They demonstrate that, for a budget or sequence length of N
  \begin{equation}
    \label{streeter_theorem}
    f(\hat{S}_{\langle N \rangle})  \geq (1-\frac{1}{e}) \max_{S \in
      \mathcal{S}}{f(S)} - \sum_{i=1}^N\epsilon_i
  \end{equation}
  assuming each control action takes equal time to execute.

  Thus the $i^{th}$ argmax in equation \eqref{approx_greedy} is chosen with some
  error $\epsilon_i = r_i$. An $\epsilon_i$ error made by classifier $\pi_i$ corresponds to the 
  classifier picking an action whose marginal gain is $\epsilon_i$ less than the maximum possible.
  Hence the performance bound on additive error greedy sequence
  construction stated in equation \eqref{streeter_theorem} can be restated as 
  \begin{equation}
    \label{regret_reduction}
    f(\hat{S}_{\langle N \rangle})  \geq (1-\frac{1}{e}) \max_{S \in
      \mathcal{S}}{f(S)} - \sum_{i=1}^Nr_i
  \end{equation}
  \todo{What original performance guarrantee? -- FIXED by removing the line on
  original performance guarantee. Had actually meant to say that the
  performance guarantee of Streeter et al. still remained over our
  classifiers.}
\end{proof}

\begin{theorem}
The sequence of squared-loss regressors ($\Re_i$) trained in Algorithm \ref{conseqopt.regression} is within at
least $(1-\frac{1}{e}) \max_{S \in
    \mathcal{S}}{f(S)} - \sum_{i=1}^N \sqrt{2(|\mathcal{V}|-1)r_{reg_i}}$ of the optimal sequence of
  classifiers $S$ from the hypothesis space of multi-class cost-sensitive
  classifiers.
\todo{Wait, we want a guarrantee on the final action performance. What exactly are you guaranteeing? Is the Langford reduction for regret or error? Also,
you never say squared error below. What is "regression" error?}
\end{theorem}
\begin{proof} \emph{(Sketch)}
  Langford et al. \cite{langford2005secoc} show that the regret reduction
  from multi-class classification to squared-loss regression
  has a regret reduction of $\sqrt{2(|k|-1)r_{reg}}$ where $k$ is the number of
  classes and $r_{reg}$ is the squared-loss regret on the underlying regression
  problem. In Algorithm \ref{conseqopt.regression} we use squared-loss regression to perform multi-class classification
  thereby incurring for each slot of the sequence a reduction regret of $\sqrt{2(|\mathcal{V}|-1)r_{reg_i}}$ where
  $|\mathcal{V}|$ is the number of actions in the control library. Theorem \ref{classification_proof} states 
  that the sequence of classifiers is within at least $f(\hat{S}_{\langle N \rangle})  \geq (1-\frac{1}{e}) \max_{S \in \mathcal{S}}{f(S)} - \sum_{i=1}^Nr_i$ 
  of the optimal sequence of classifiers. Plugging in the regret reduction from \cite{langford2005secoc} we get the result
  that the resulting sequence of regressors in Algorithm \ref{conseqopt.regression} is within at least
  $(1-\frac{1}{e}) \max_{S \in \mathcal{S}}{f(S)} - \sum_{i=1}^N
  \sqrt{2(|\mathcal{V}|-1)r_{reg_i}}$ of the optimal sequence of multi-class
  cost-sensitive classifiers.

\end{proof}

%%%%%%%%%%%%%%%%%%%%%%%%%%%%%%%%%%%%%%%%%%%%%%%%%%%%%%%%%%%
\section{Case Studies}
\label{experiments}

\subsection{Robot Manipulation Planning via Contextual Control Libraries}

We demonstrate the application of \textsc{ConSeqOpt} for manipulation planning
on a 7 degree of freedom manipulator. 

Recent work \cite{ratliff2009chomp,jetchev2009trajectory} has shown
that by relaxing the hard constraint of avoiding obstacles into a
soft penalty term on collision, simple local optimization techniques can
quickly lead to smooth, collision-free trajectories suitable
for robot execution. Often the default initialization trajectory seed is a simple
straight-line initialization in joint space \cite{ratliff2009chomp}. This
heuristic is surprisingly effective in many environments, but suffers from
local convergence and may fail to find a trajectory when
one exists. In practice, this may be tackled by providing cleverer
initialization seeds \cite{jetchev2009trajectory, zucker2009proposal}. While
these methods reduce the chance of falling into local
minima, they do not have any alternative plans should the chosen
initialization seed fail. A contextual ranking of a library of initialization trajectory seeds
can provide feasible alternative seeds
should earlier choices fail. Proposed initialization trajectory seeds can be
developed in many ways including human demonstration \cite{ratliff2007} or use
of a slow but complete planner\cite{kuffner2000}.

For this experiment we attempt to plan a trajectory to a pre-grasp pose over
the target object in a cluttered environment using the local optimization
planner CHOMP \cite{ratliff2009chomp} and minimize the total planning and
execution time of the trajectory. A training dataset of $|D|=310$ environments
and test dataset of $212$ environments are generated. Positions of the
target object and obstacles on the table are randomly assigned. To
populate the control library, we consider initialization
trajectories that move first to an ``exploration point" and then to the
goal. The exploration points are generated by randomly perturbing the midpoint
of the original straight line initialization in joint space. The resulting
initial trajectories are then piecewise straight lines in joint space from the
start point to the exploration point, and from the exploration point to the
goal. Half of the seed trajectories are prepended with a short path to start
from a left-arm configuration, and half are in right-arm configuration. This
is because the local planner has a difficult time switching between
configurations, while environmental context can provide a lot of information
about which configuration to use. $30$ trajectories are generated and form our
action library $\mathcal{V}$. Figure \ref{trajectoryInitial} shows an example set
for a particular environment. Notice that in this case the straight-line
initialization of CHOMP goes through the obstacle and therefore CHOMP has a
difficult time finding a valid trajectory using this initial seed.

\begin{figure}
  \centering
  \mbox{
    \subfloat[The default straight-line initialization of CHOMP is marked in
    orange. Notice this initial seed goes straight through the obstacle and
    causes CHOMP to fail to find a collision-free
    trajectory.]{\label{trajectoryInitial}\includegraphics[width=0.45\textwidth]{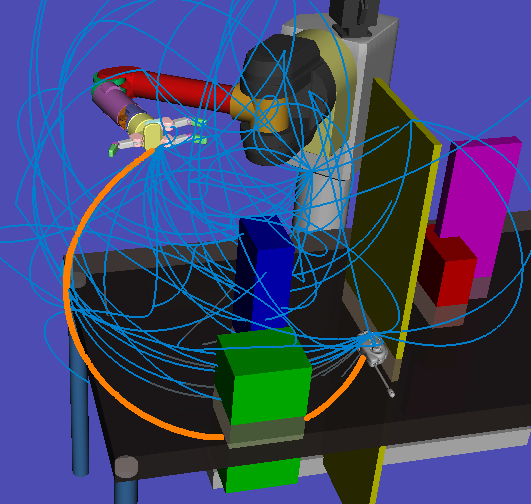}}
  }
  \mbox{
    \subfloat[The initialization seed for CHOMP found using \textsc{ConSeqOpt}
    is marked in orange. Using this initial seed CHOMP is able to find a
    collision free path that also has a relatively short execution
    time.]{\label{trajectoryFinal}\includegraphics[width=0.45\textwidth]{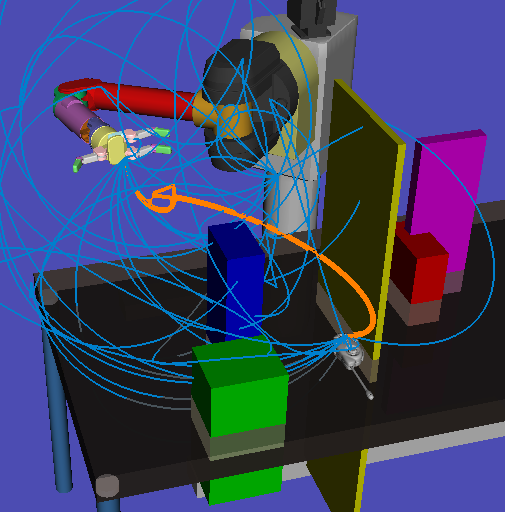}}
  } \\  
\caption{CHOMP initialization trajectories generated as control
    actions for \textsc{ConSeqOpt}. Blue lines trace the end effector path
    of each trajectory in the library. Orange lines in each image trace the
    initialization seed generated by the default straight-line approach and by
    \textsc{ConSeqOpt}, respectively.}
  \label{fig:trajectoryLibrary}
\end{figure}

In our results we use a small number ($1-3$) of slots in our sequence to ensure
the overhead of ordering and evaluating the library is small. When CHOMP fails to find a collision-free trajectory
for multiple initializations seeds, one can always fall back on
slow but complete planners. \todo{Not sure the previous sentence is necessary or helps. Consider removing. --FIXED}
Thus the contextual control sequence's role is to quickly evaluate a few good options and choose the initialization trajectory
that will result in the minimum execution time. We note that in our
experiments, the overhead of ordering and evaluating the library is negligible
as we rely on a fast predictor and features computed as part of
the trajectory optimization, and by choosing a small sequence length we can
effectively compute a motion plan with expected planning time under $0.5$s.
We can solve most manipulation problems that arise in our manipulation
research very quickly, falling back to initializing the trajectory
optimization with a complete motion planner only in the most difficult of
circumstances. 

For each initialization trajectory, we calculate $17$ simple feature values
which populate a row of the feature matrix $\mathbf{X_i}$: length of
trajectory in joint space; length of
trajectory in task space, the xyz values of the end effector position at the
exploration point (3 values), the distance field \todo{Is it clear what a
  distance field is? -- We cant explain distance field without breaking
  phrasing and taking up a lot of space for an experimental detail.} values used by CHOMP at the
quarter points of the trajectory (3 values), joint values of the first 4
joints at both the exploration point (4 values) and the target pose (4 values),
and whether the initialization seed is in the same left/right kinematic arm
configuration \todo{What does left/right arm mean? -- FIXED: added kinematic to show it's a common term} as the target pose. During training time, we evaluate each
initialization seed in our library on all environments in the training set,
and use their performance and features to train each regressor $\Re_i$ in
\textsc{ConSeqOpt}. At test time, we simply run Algorithm \ref{conseqopt.regression} without
the training step to produce $Y_{\Re_1,\ldots,\Re_{N}}$ as the sequence of
initialization seeds to be evaluated. Note that while the first regressor uses only
the 17 basic features, the subsequent regressors also include the difference
in feature values between the remaining actions and the actions chosen
by the previous regressors. These difference features improve the algorithm's
ability to consider trajectory diversity in the actions chosen.

We compare \textsc{ConSeqOpt} with two other methods of ranking the
initialization library: a random ordering of the actions, and an ordering
by sorting the output of the first regressor. Sorting by the first regressor
is functionally the same as maximizing the absolute benefit rather than the
marginal benefit at each slot. We compare both the number of
CHOMP failures as well as the average execution time of the final
trajectory. For execution time, we assume the robot can be actuated at
$1$ rad/second for each joint and use the shortest trajectory generated using the
$N$ seeds ranked by \textsc{ConSeqOpt} as the performance. If we fail to find a
collision free trajectory and need to fall back to a complete planner
(RRT \cite{kuffner2000} plus trajectory optimization), we apply a maximum
execution time penalty of 40 seconds due to the longer computation time and
resulting trajectory.

The results over $212$ test environments are summarized in Figure \ref{fig:trajectoryResults}.  With only simple straight line initialization, CHOMP is unable to find a collision free trajectory in 162/212 environments, with a resulting average
execution time of 33.4s. 
While a single regressor ($N=1$) can reduce the number of CHOMP failures from
162 to 79 and the average execution time from 33.4s to 18.2s, when we extend
the sequence length, \textsc{ConSeqOpt} is able to reduce both metrics faster
than a ranking by sorting the output of the first regressor. This is because
for $N>1$, \textsc{ConSeqOpt} chooses a primitive that provides the maximum
marginal benefit, which results in trajectory seeds that have very different
features from the previous slots' choices. Ranking by the absolute benefit
tends to pick trajectory seeds that are similar to each other, and thus are
more likely to fail when the previous seeds fail. At a sequence length of 3,
\textsc{ConSeqOpt} has only 16 failures and an average execution time of 3
seconds. \textbf{A 90\% improvement in success rate and a 75\% reduction in execution time.} Note that planning times are generally negligible compared to execution times
for manipulation hence this improvement is significant. Figure \ref{trajectoryFinal} shows the initialization seed found by
\textsc{ConSeqOpt} for the same environment as in
Figure \ref{trajectoryInitial}. Notice that this seed avoids collision with
the obstacle between the manipulator and the target object enabling CHOMP to
produce a successful trajectory.

\begin{figure}[ht]
\begin{center}
\includegraphics[width=0.45\textwidth]{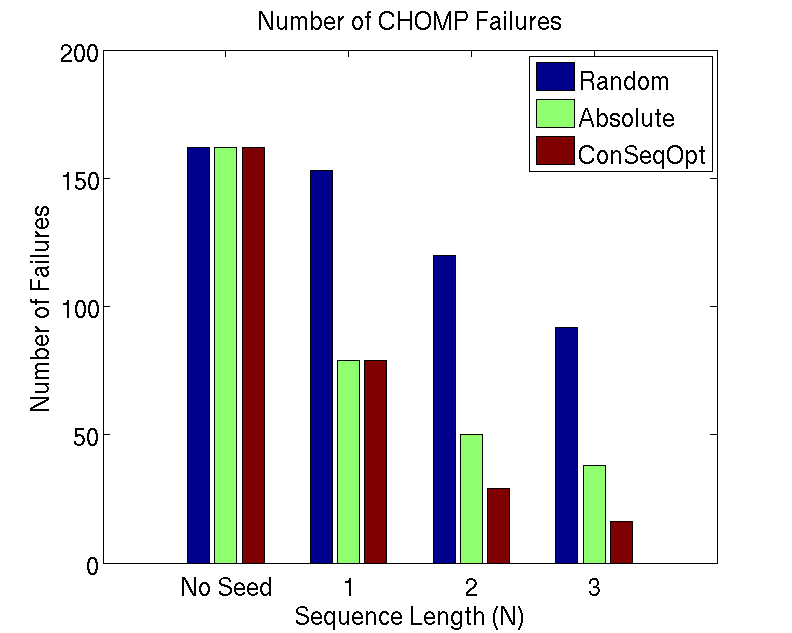}
\includegraphics[width=0.45\textwidth]{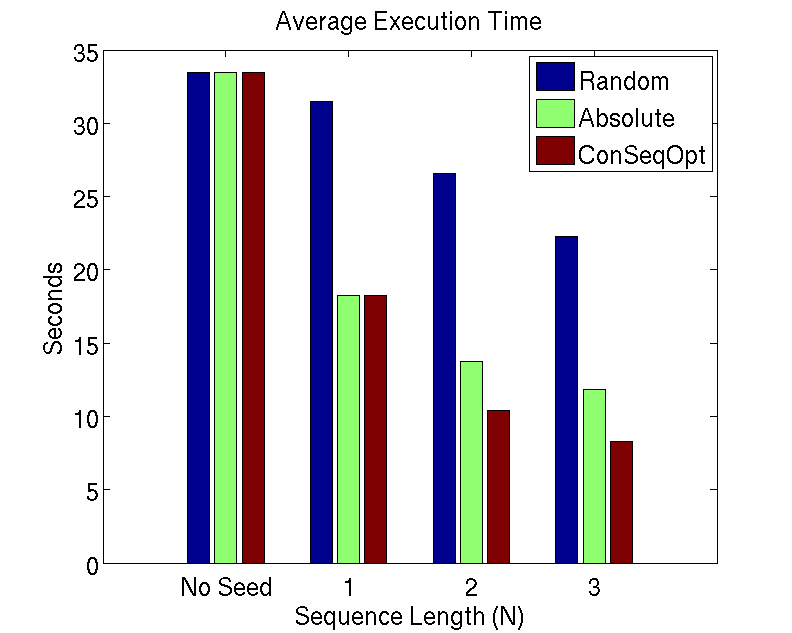}
\caption{Results of \textsc{ConSeqOpt} for manipulation planning
    in 212 test environments. The top image shows the number of CHOMP
    failures for three different methods after each slot in the
    sequence. \textsc{ConSeqOpt} not only significantly reduces the number of
    CHOMP failures in the first slot, but also further reduces the failure
    rate faster than both the other methods when the sequence length is
    increased. The same trend is observed in the bottom image, which shows the
    average execution time of the chosen trajectory. The `No Seed' column
    refers to the straight-line heuristic used by the original CHOMP
    implementation}
\label{fig:trajectoryResults}
\end{center}
\end{figure}

\subsection{Mobile Robot Navigation}
An effective means of path planning for mobile robots is to sample a budgeted
number of trajectories from a large library of feasible trajectories and traverse
the one which has the lowest cost of traversal for a small portion and repeat the process again. The sub-sequence
of trajectories is usually computed offline \cite{green2006paths,erickson2009survivability}. Such methods are widely used in modern, autonomous
ground robots including the two highest placing teams for DARPA Urban
Challenge and Grand Challenge
\cite{urmson2008,montemerlo2008,urmson2006,thrun2006}, LAGR
\cite{jackel2006darpa}, UPI \cite{bagnell2010crusher}, and
Perceptor \cite{kelly2006perceptor} programs. We use \textsc{ConSeqOpt} to
maximize this function and generate trajectory sequences taking the current
environment features. 
% 
%  We prove in
% \cite{isrrAnonymous} that minimizing the objective function $f$
% (equation.\ref{f_path_planning}) is a monotone, submodular function over
% sequences of trajectories. Thus

Figures \ref{overhead_color} and \ref{cost_map} shows a section of Fort
Hood, TX and the corresponding robot cost-map respectively. We
simulated a robot traversing between various random starting and goal
locations using the maximum-discrepancy trajectory \cite{green2006paths}
sequence as well as sequences generated by \textsc{ConSeqOpt} using
Algorithm \ref{conseqopt.classification}.  A texton library
\cite{winn2005texton} of $25$ k-means cluster centers was computed for the
whole overhead map. At run-time
the texton histogram for the image patch around the robot was used as
features. Online linear support vector machines (SVM) with slack re-scaling
\cite{scholkopf2002learning} were used as the
cost-sensitive classifiers for each slot. We report a $9.6$\% decrease over $580$ runs using $N=30$ trajectories in the
cost of traversal as compared to offline precomputed trajectory sequences
which maximize the area between selected trajectories \cite{green2006paths}. Our approach is able to choose which trajectories to use at each step based on the appearance of terrain (woods, brush, roads, etc.)
As seen in Figure \ref{robot_on_map} at each time-step \textsc{ConSeqOpt} the
trajectories are so selected that most of them fall in the empty space around
obstacles.

\begin{figure}[!htbp]
  \centering
  \mbox{
    \subfloat[Overhead color map of portion of Fort Hood,
    TX]{\label{overhead_color}\includegraphics[width=0.2\textwidth]{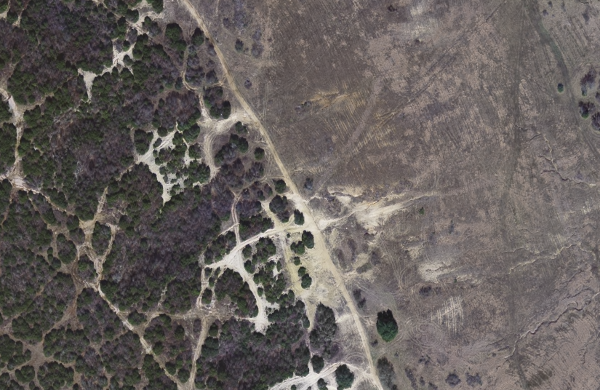}}
  }
  \mbox{
    \subfloat[Cost map of corresponding
    portion]{\label{cost_map}\includegraphics[width=0.2\textwidth]{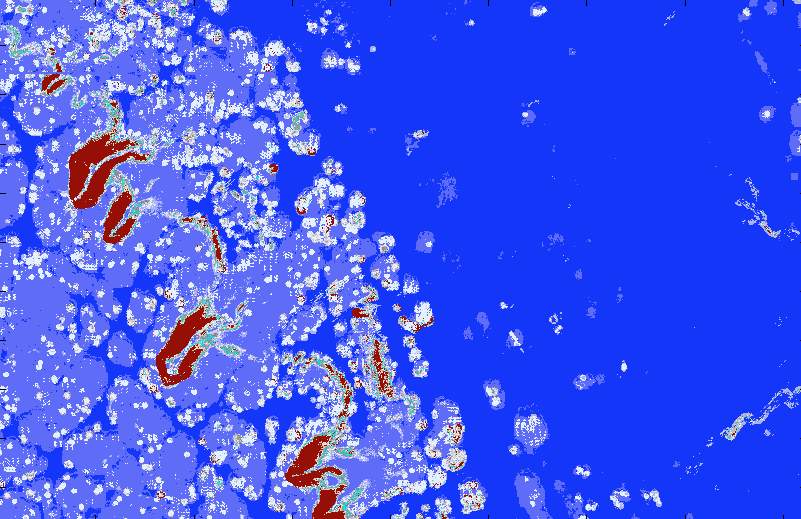}}
  } 
  \mbox{
  \subfloat[Robot traversing the map using \textsc{ConSeqOpt} generating
  trajectory sequences which try to avoid obstacles in the
  vicinity]{\label{robot_on_map}\includegraphics[width=0.4\textwidth]{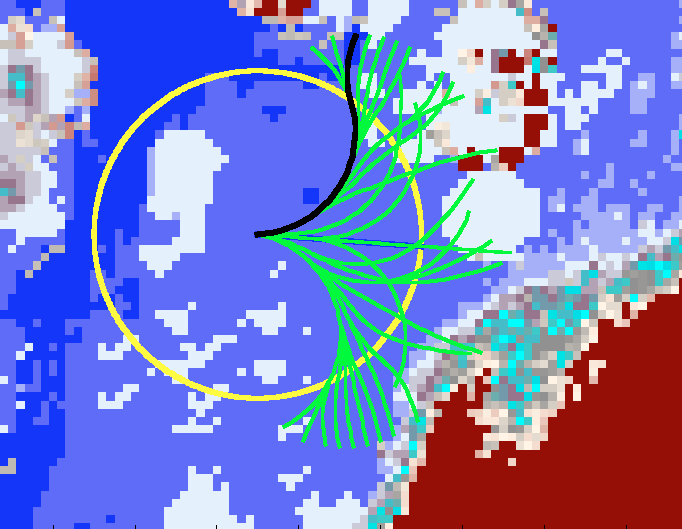}}
  }
	  % \mbox{
	  % \subfloat[Cost of path traversal of \textsc{ConSeqOpt} is
	  %   lower than Green-Kelly by $9.6$\% over 580 runs of the robot
	  %   between random start and goal points on the map using $30$ trajectories.]{\label{path_results}\includegraphics[width=0.3\textwidth]{images/path_results.png}}
	  % \label{path_all}
	  %}
\end{figure}

%%%%%%%%%%%%%%%%%%%%%%%%%%%%%%%%%%%%%%%%%%%%%%%%%%%%%%%%%%%
% \section{Conclusion}
% \label{conclusion}

% Our approach improves both the performance of control libraries and increases
% robustness. The near-optimal ordering of control
% primivites can be used to keep trying to complete the mission even if initial
% primitives fail while having on-average performance guarantees. Such behaviour
% is crucial to successful deployment of robots in the real world.

% \textsc{ConSeqOpt} is a generic contextual optimization technique and has
% application to other domains like advertisement placement on web pages.

% Note that when features for regressors for subsequent slots are
% dependent on features of chosen primitives in earlier slots the objectives are
% likely to be adaptive submodular \cite{golovin2010adaptive}. In future work we
% intend to explore how the implications of adaptive submodularity on contextual
% sequence optimization.

%% Use plainnat to work nicely with natbib. 

\bibliographystyle{plainnat}
\bibliography{../references}

\end{document}